\documentclass[conference]{IEEEtran}

\usepackage{cite}
\usepackage{amsmath,amssymb,amsfonts,amsthm}
\usepackage{algorithmic}
\usepackage{algorithm}
\usepackage{graphicx}
\usepackage{textcomp}
\usepackage{xcolor}
\usepackage{flushend}
\usepackage{tikz}
\usetikzlibrary{shapes,arrows,positioning}
\usepackage{booktabs}   
\usepackage{threeparttable} 
\def\BibTeX{{\rm B\kern-.05em{\sc i\kern-.025em b}\kern-.08em
    T\kern-.1667em\lower.7ex\hbox{E}\kern-.125emX}}

\newtheorem{theorem}{Theorem}[section]

\newtheorem{definition}[theorem]{Definition}

\title{Causal SHAP: Feature Attribution with Dependency Awareness through Causal Discovery}

\author{\IEEEauthorblockN{Woon Yee Ng\IEEEauthorrefmark{1},
Li Rong Wang\IEEEauthorrefmark{1}, 
Siyuan Liu\IEEEauthorrefmark{1}, and
Xiuyi Fan\IEEEauthorrefmark{2}\IEEEauthorrefmark{1}}
\IEEEauthorblockA{
\IEEEauthorrefmark{1}College of Computing and Data Science,\\
\IEEEauthorrefmark{2}Lee Kong Chian School of Medicine,\\
Nanyang Technological University, Singapore}}

\begin{document}

\maketitle

\begin{abstract}
Explaining machine learning (ML) predictions has become crucial as ML
models are increasingly deployed in high-stakes domains such as
healthcare. While {\em SHapley Additive exPlanations (SHAP)} is widely
used for model interpretability, it fails to differentiate between
causality and correlation, often misattributing feature importance
when features are highly correlated. We propose {\em Causal SHAP}, a
novel framework that integrates causal relationships into feature
attribution while preserving many desirable properties of SHAP. By
combining the {\em Peter-Clark (PC)} algorithm for causal discovery and
the {\em Intervention Calculus when the DAG is Absent (IDA)} algorithm
for causal strength quantification, our approach addresses the
weakness of SHAP. Specifically, Causal SHAP reduces
attribution scores for features that are merely correlated with the
target, as validated through experiments on both synthetic and
real-world datasets. This study contributes to the field of
Explainable AI (XAI) by providing a practical framework for
causal-aware model explanations. Our approach is particularly valuable
in domains such as healthcare, where understanding true causal
relationships is critical for informed decision-making. 
\end{abstract}

\section{Introduction}
As machine learning models become more sophisticated, their complexity creates a black-box conundrum~\cite{vilone2021notions}, where the lack of transparency undermines trust. As models are deployed in high-stakes decision-making tasks, such as autonomous driving and medical diagnosis, and face new regulatory requirements like the EU AI Act, the need for interpretable AI has become essential—not only for legal compliance but also for ensuring alignment with societal standards.

One of the most widely used machine learning interpretability methods is SHapley Additive exPlanations (SHAP)~\cite{lundberg2017unified}. SHAP is inspired by Shapley values \cite{shapley1953value}, a game theory concept that quantifies each player's contribution in a cooperative setting. A key limitation of SHAP and its common variations is their inability to capture a fundamental aspect of real-world systems: causality \cite{frye2020asymmetric}. While SHAP makes the correlations learned by predictive ML models transparent, it disregards causal relationships among model inputs. This limitation arises from its assumption of feature independence, which can lead to misattributed feature importance \cite{aas2021explaining} and the generation of impossible feature combinations \cite{hooker2019please}.

For example, in Figure~\ref{fig:open-exp}, SHAP (left) assigns importance based on correlations, suggesting that all features have direct edges to lung cancer, implying that they all directly contribute to the ``lung cancer'' prediction. However, domain knowledge (right) reveals that not all features have a direct path to lung cancer and that some features are dependent on others. As a result, SHAP misattributes importance and fails to account for causal relationships between features. This limitation is particularly problematic in healthcare, where incorrect feature attribution could lead to misguided therapeutic decisions.

\begin{figure}[H]
    \centering
    \includegraphics[width=1\linewidth]{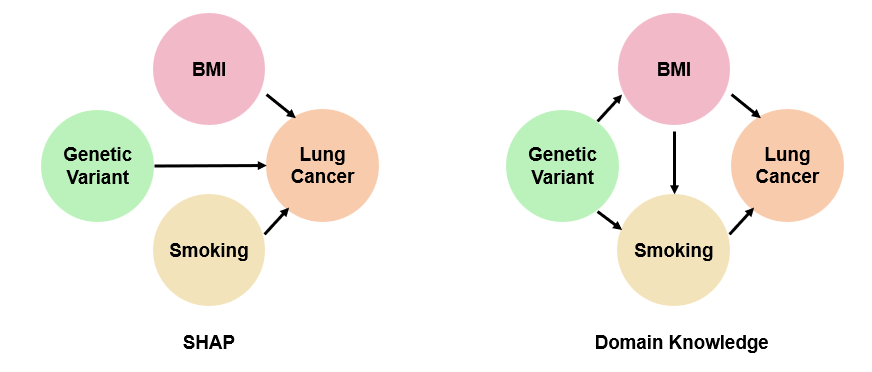}
    \caption{On the left, we see the standard SHAP calculation of feature attribution. All three cancer risk factors, BMI, Genetic Variant, and Smoking are considered equally. However, studies have shown that Smoking and BMI both act as mediators in the causal graph \cite{zhou2021causal}. Thus, to correctly quantify their importance, we must account for these causal relationships in feature attribution.}
    \label{fig:open-exp}
\end{figure}

Recent approaches have attempted to address this limitation by integrating causality into SHAP, including On-Manifold Shapley Values \cite{frye2020shapley}, Asymmetric Shapley Values (ASV) \cite{frye2020asymmetric}, Causal Shapley Values \cite{heskes2020causal}, and Shapley Flow \cite{wang2021shapley}. However, these methods still face challenges in efficiently handling feature dependencies and causality. They either require a complete causal graph for the algorithms to function or suffer from exponentially increasing computational complexity as the number of features grows.

In this paper, we propose a novel Causal SHAP framework that explains how each feature in the dataset contributes to the model's prediction while respecting the causal relationships within the data. Specifically, our approach integrates causal relationships into SHAP calculations by employing the Peter-Clark (PC) algorithm \cite{spirtes2001causation} for causal edge discovery and the Intervention Calculus when the DAG is Absent (IDA) algorithm \cite{maathuis2010predicting} for causal strength quantification. The PC algorithm is a constraint-based method that uses statistical tests to uncover causal relationships in data \cite{spirtes2001causation}. IDA builds upon the directed acyclic graphs (DAGs) inferred by PC to estimate causal effects using intervention calculus \cite{maathuis2010predicting}. This combination allows us to model how features influence one another through directed causal paths, rather than treating all correlations equally. Both PC and IDA are well-established and widely used techniques in causal discovery.

To validate the effectiveness of our approach, we evaluated Causal SHAP on both synthetic data with known causal structures and real-world datasets. We demonstrate that our method not only provides more accurate feature attributions but also better reflects the causal structure of the underlying datasets. Extensive experiments were conducted to highlight the advantages of our method.

In summary, the contributions of this paper are as follows:
\begin{enumerate}
\item A novel framework for integrating causality into SHAP using the PC algorithm for causal discovery between features and the IDA algorithm for quantifying causal strength.
\item Theoretical results showing that the proposed approach preserves the three key desired properties of SHAP: local accuracy, missingness, and consistency.
\item Experimental validation of Causal SHAP's effectiveness through the insertion test and the RMSE between the computed SHAP scores and the ground truth.
\end{enumerate}

\section{Related Work}
Recent work has attempted to integrate causality into SHAP through methods such as On Manifold Shapley Values \cite{frye2020shapley}, Asymmetric Shapley Values (ASV) \cite{frye2020asymmetric}, Causal Shapley values \cite{heskes2020causal} and Shapley Flow \cite{wang2021shapley}. On Manifold Shapley Values focuses on local explanation and attempts to provide more accurate explanations when features are dependent on each other. It does this by modifying the Shapley value function to condition out-of-coalition features on in-coalition features. However, the computational time scales exponentially as the number of features increases. Causal Shapley Values \cite{heskes2020causal} proposed a node-based approach that incorporates Pearl's do calculus to split attribution between parents and children in the a priori causal graph, allowing the method to capture both direct and indirect effects, but requires the complete causal graph to implement. When only partial causal information is available, Asymmetric Shapley Values (ASV) \cite{frye2020asymmetric} provided a theoretical framework that relaxes the Axiom of Symmetry to allow non-uniform weighting of feature importance to integrate causality into model explainability, therefore providing a more nuanced explanation that can prioritise causal ancestor and descendant. Shapley Flow \cite{wang2021shapley} further advanced this by assigning attribution to edges instead of nodes, and unified Shapley value based methods including ASV. 

While these methods address the causal aspect, another key challenge is handling feature dependencies in practice. Popular implementations of SHAP variants (KernelSHAP \cite{lundberg2017unified}, Interventional TreeSHAP \cite{lundberg2020local}, DeepSHAP \cite{lundberg2017unified}) assume feature independence for computational efficiency. Although TreeSHAP offers a path-dependent variant that can handle dependencies, it is limited to tree-based models and produces biased estimates \cite{amoukou2022accurate}. To address feature dependencies, several approximation methods have been proposed, including approaches using multivariate Gaussian distributions with Kernel Estimation \cite{janzing2020feature}, as well as methods based on Gaussian distributions, Gaussian copula distributions, or empirical distributions \cite{aas2021explaining} to estimate conditional expectations.

However, as these frameworks are built on Shapley Values, they face exponential complexity as the number of features increases \cite{michalak2013efficient}. Hence, various approximation techniques have been proposed, including sampling-based approaches and model-specific optimizations \cite{lundberg2017unified}. Despite all these techniques, the efficient handling of both causality and feature dependencies remains an open challenge.

Existing causal attribution methods like ASV, Causal Shapley Values, and Shapley Flow typically assume access to a causal graph from domain experts, while recent advances in causal discovery offer promising alternatives for automatically inferring these relationships. However, these methods have not been integrated with SHAP frameworks.
\section{Background}
In this section, we briefly review key concepts used in this work, namely, Shapley values from cooperative game theory, SHAP for model explanations, Peter-Clark (PC) algorithm for causal discovery, and a causal effect estimation method named ``Intervention calculus when the DAG is absent'' (IDA).
\paragraph{Shapley Values}
In cooperative game theory, the Shapley value of a player $i$ represents their contribution to the payout and is computed as the weighted sum over all possible feature value combinations:
\begin{equation}
    \phi_i = \sum_{S \subseteq N \setminus \{i\}} \frac{|S|!(n-|S|-1)!}{n!}[v(S \cup \{i\}) - v(S)].
\end{equation}
Here, $\phi_i$ represents the contribution of player $i$ to the prediction; $N$ is the set of all players; $n$ is the number of players; $S$ is a subset of players excluding player $i$; $v(S)$ is the value function defined as
\begin{equation}
  \label{eqn:base_value}
  v(S) = E\left[f(X)|X_S = x_S \right],
\end{equation}
where $f$ is the predictive model; $X$ is the vector of input features (players); and $X_S = x_S$ means we are conditioning on knowing the value in the subset of players $S$ for this expectation calculation. 

\paragraph{SHAP}
SHAP is a method used to explain how each feature contributes to a
prediction in a prediction model, using Shapley values. Although SHAP
is model-agnostic, it makes the assumption of feature independence
\cite{aas2021explaining}, which is rare in real-world data sets. The
SHAP documentation states that the interventional expectation  
\begin{equation}
  v(S) = E[f(X)|do(X_S = x_S)]
  \label{eqn:SHAP_Interventional}
\end{equation} 
is used to approximate the conditional expectation given in
Eqn.~\ref{eqn:base_value}, while computing SHAP values. The
do-operator, $do(X_S)$, introduced by Pearl \cite{pearl2000models},
represents an intervention that sets variables $X_S$ to specific
values $x_S$ while ignoring any possible influence of the set of
players $N \setminus S$ on the set of interventions. This approximation is
used for computational efficiency, as computing the true conditional
expectation would require estimating complex conditional distributions
and would scale exponentially with the number of features. 

\paragraph{Peter-Clark (PC)}
PC is a widely used constraint-based causal discovery
algorithm \cite{spirtes2001causation} that constructs a causal graph
in two phases: (1) initialization of the graph skeleton and (2) identification of the edge orientation, as illustrated in Figure~\ref{fig:PC_illustration}.
\begin{figure}[H]
    \centering
    \includegraphics[width=0.75\linewidth]{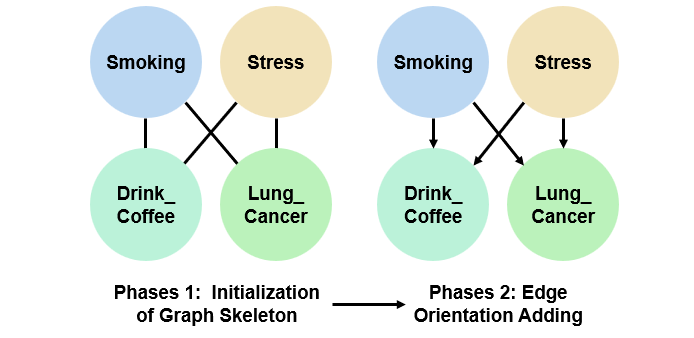}
    \caption{An example of two phases in PC algorithm}
    \label{fig:PC_illustration}
\end{figure}
Starting with a fully connected undirected graph (assuming full
interactions between all features), phase 1 systematically removes
edges between conditionally independent pairs based on statistical
tests, i.e., for each pair of variables found to be conditionally
independent given a set of other variables (called the separating
set), the edge between the pair is removed. These tests are performed
with increasingly large conditioning sets.
The output of the first phase is a graph skeleton that shows causal
relationships between pairs of variables. In the second phase, PC
first identifies v-colliders \cite{spirtes2001causation} where
non-adjacent variables share a  common neighbor not in their
separating set, then applies orientation rules to form a completed
partially directed acyclic graph (CPDAG).

\paragraph{IDA}
Using the CPDAG computed with PC, IDA finds the parents of each
variable and estimates the causal effects between all pairs of parents and children. As multiple DAGs can be extracted from a CPDAG, IDA uses
Pearl’s do-calculus to obtain a multi-set of possible causal effects,
as illustrated in Fig.~\ref{fig:IDA}. For each DAG, IDA
estimates non-negative edge weights representing causal strength using regression for all edges \cite{kalisch2012causal}. 

\begin{figure}[H]
    \centering
    \includegraphics[width=0.80\linewidth]{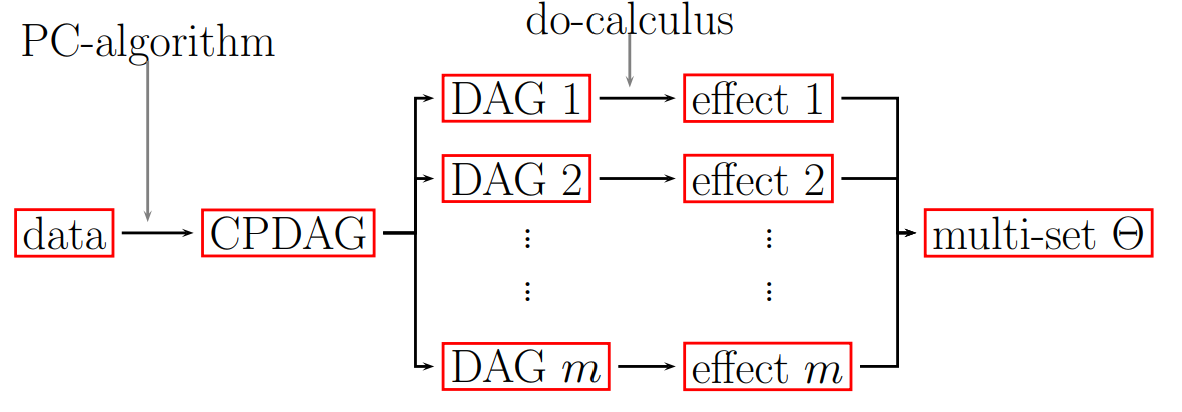}
    \caption{An illustration of IDA (adapted from \cite{maathuis2010predicting}).}
    \label{fig:IDA}
\end{figure}

\section{Methodology}
\subsection{Problem Setting}
To infuse causality into SHAP, we introduce Causal SHAP to attribute a
model prediction to its input features, while respecting known causal
relationships between variables.

Considering the following {\em feature attribution} problem: Given a
prediction model $f: \mathcal{X} \rightarrow \mathbb{R}$ where
$\mathcal{X} \subseteq \mathbb{R}^n$ and an instance $x = (x_1,
\ldots, x_n) \in \mathcal{X}$, derive the contribution of each feature
to the models prediction $f(x)$ (Definition~\ref{def:feat_attr}). 
\begin{definition}[Feature Attribution] 
Given an instance $x$, a feature attribution method $\zeta$ is a
mapping $\zeta(x;f):\mathcal{X}\rightarrow\mathbb{R}^n$, with each
$\phi_i=\zeta(x;f)_i$ quantifying the contribution of the feature $i$ to
$f(x)$. 
\label{def:feat_attr}
\end{definition}

\begin{figure*}
    \centering
    \includegraphics[width=0.75\linewidth]{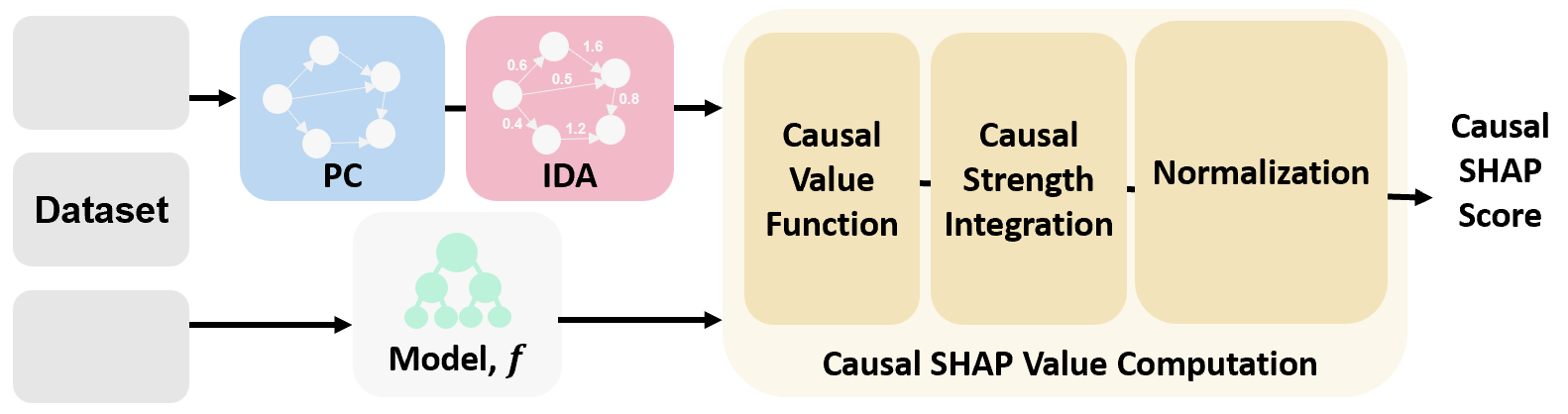}
    \caption{Causal SHAP Framework: A pipeline for integrating causality into SHAP feature attribution. The framework consists of three main components: (1) PC algorithm for discovering causal relationship between features. (2) IDA algorithm for estimating the strength of causal effect. (3) Causal SHAP value computation to incorporate both causal structure and strength information through Causal Value Function, Causal Strength Integration and followed by a normalization step to maintain additivity. This approach combines causal discovery with feature attribution while preserving SHAP's desirable properties.}
    \label{fig:causal SHAP}
\end{figure*}

In this work, we model the causal relationships between variables
using a DAG, $G = (V, E)$ where: $V = \{1, \ldots, n \}$ is the set of
features and $E \subseteq V \times V$ represents causal relationships
between features. For each feature $i \in V$, $\text{PA}_i = \{j: (j,
  i) \in E\}$ denotes its parent set and $W_i$ its total causal effect on the prediction target $Y$.
Note that $W_i$ is computed using causal effects between edges found with the IDA algorithm.


Our objective is to compute feature attributions $\{\phi_i\}_{i=1}^n$
that follow the SHAP approach,
while satisfying the key properties: local accuracy,
missingness and consistency (see Section~\ref{sec:theory} for their
definitions) and reflect both:
\begin{itemize}
     \item Direct effects: Causal strength through edges $(i, Y)$. 
     \item Indirect effects: Causal strength through paths $i \rightsquigarrow Y$.
\end{itemize}
The key challenge lies in effectively computing these attributions
while accounting for the causal structure in $G$.
\subsection{Method}

We propose a two-step modification to SHAP to account for causality
in its feature attribution as illustrated in Figure~\ref{fig:causal SHAP}.
Firstly, we introduce a new value function ({\em Causal Value
  Function}) to better approximate the Shapley value function
(Equation~\ref{eqn:base_value}) than SHAP's interventional expectation
(Equation~\ref{eqn:SHAP_Interventional}) by preventing the generation
of impossible data points in its sampling-based calculation. Next, we
consider causal strength integration to assign different strengths to
different causal edges, effectively discounting feature attribution values with ``causal strength''. 

The causal value function is inspired by causal intervention, in which one observes changes in the output while manipulating the input. In our
context, we incorporate Pearl's do-operator, which isolates the causal
effect of a subset of variables by fixing their values, and sampling the
other variables outside of the subset, with respect to the causal
graph. The difference between our method
(Eqn~\ref{eqn:causal_SHAP_value}) from the original SHAP
(Eqn~\ref{eqn:SHAP_Interventional}) is that we  
sample the out-of-coalition features with respect to the causal
relationship, hence prevent the generation of impossible data points
that would not happen in real-life. The causal relationship is based
on the causal graph computed using the PC algorithm:
\begin{equation}
\begin{split}
v_{\text{c}}(S) = \mathbb{E}\Big[f(X) \mid \text{do}(X_S = x_S),\\
\quad X_{\bar{S}} \sim \Pr(X_{\bar{S}} \mid \text{do}(X_S = x_S))\Big].
\end{split}
\label{eqn:causal_SHAP_value}
\end{equation}
Here, $do(X_S = x_S)$ represents Pearl's do-operator. $\bar{S} = \{i
  \in {1, \ldots, n} \mid i \notin S\}$ denotes the set of variables
not in $S$. $\Pr(X_{\bar{S}} \mid \text{do}(X_S = x_S))$ represents
the distribution of variables not in $S$, given the intervention on
variables in $S$. 

The distribution $\Pr(X_{\bar{S}} \mid \text{do}(X_S = x_S))$ is
determined by the causal mechanisms, specifically, it depends on
whether these variables have parents in the causal graph. For a
variable $X_i \in X_{\bar{S}}$ that has no parents in the causal
graph, its distribution remains unchanged: 
\begin{equation}
X_i \sim \Pr(X_i),
\end{equation}
where $\Pr(X_i)$ is the original marginal distribution of $X_i$. For a
variable $X_i \in X_{\bar{S}}$ that has parents $\text{PA}_i$ in the
causal graph, its distribution is conditioned on the values of its
parents: 
\begin{equation}
X_i \sim N(\mu_i,\sigma_i),
\end{equation}
where $\mu_i$ is estimated using linear regression of $X_i$ based on its parents $\text{PA}_i$, and $\sigma_i$ is estimated from the regression residuals. When computing the causal value function $v_{\text{c}}(S)$, we sample the values of variables in $X_{\bar{S}}$ from their respective distributions, as described above. This ensures that the values of variables not in $S$ are consistent with the causal relationships in the graph, given the interventions performed on variables in $S$.

The total causal effect $W_i$ for each feature $i$ on target $Y$ is computed by analyzing all possible paths in the IDA multi-set output from feature $i$ to $Y$. For a given path $p$, the path-specific effect $W_i^p$ on target $Y$ is calculated as the product of individual edge weights (mean causal effects) along that path:
\begin{equation}
W_i^p = \prod_{(j,k) \in p} w_{jk},
\end{equation}
where $w_{jk}$ represents the edge weight between nodes $j$ and $k$ retrieved from IDA algorithm.
The total causal effect $W_i$ for feature $i$ is then computed as the sum of effects across all paths:
\begin{equation}
W_i = \sum_{p \in P_i} W_i^p,
\end{equation}
where $P_i$ is the set of all simple paths from feature $i$ to the
target $Y$.\footnote{While PC produces a CPDAG, $P_i$ represents
  directed paths in DAGs from the IDA equivalence class, which are
  finite due to acyclicity of the graph.} Then, we integrate causal strength into
SHAP to ensure that it accounts for information from the causal
graph. The causal weight factor $\gamma_{i}$ is defined as: 
\begin{equation}
    \gamma_{i} = \frac{|W_i|}{\sum_{j \in N} |W_j|},
\end{equation}
where $W_i$ represents the total causal effect of feature $i$ on the
target variable, computed using the IDA algorithm
\cite{maathuis2010predicting}. Hence, our novel {\em Causal SHAP} value is
defined as: 
\begin{equation}
  \label{eqn:phi_i_c}
    \phi_i^{c} = \sum_{S \subseteq N \setminus \{i\}} w(S,i)[v_{c}(S \cup \{i\}) - v_{c}(S)],
\end{equation}
where
\begin{equation*}
    w(S,i) = \frac{|S|!(n-|S|-1)!}{n!}\gamma_{i}.
\end{equation*}

Lastly, we normalize the causal SHAP values with:
\begin{equation}
    \phi_i^{n} = \phi_i^{c} \times \frac{f(x) - E[f(X)]}{\sum_{j=1}^n \phi_j^{c}}.
\end{equation}

\section{Theoretical Evaluation}
\label{sec:theory}

SHAP is known for its good properties: local accuracy, missingness, and
consistency. Together, they imply the four Shapley value axioms 
\cite{lundberg2017unified}.
\begin{itemize}
    \item Local Accuracy: The explanation for an instance should
      reflect the prediction probability for that instance as given by
      the predictor. 
    \item Missingness: Any instance missing a feature-value should
      assign a zero attribution value to that missing feature. 
    \item Consistency: For any two prediction models where one has a
      larger change in prediction with the removal of a feature than
      another. Then the resulting explanation for that feature should
      be larger.
\end{itemize}
We show that the proposed normalised Causal SHAP maintain these properties as follows, starting with the Local Accuracy

\begin{theorem}[Local Accuracy]
For any model $f$ and features $X$, the normalized causal SHAP values satisfy:
    \begin{equation}
        \sum_{i \in N} \phi_i^{\text{n}} = f(x) - \mathbb{E}[f(X)],
    \end{equation}
where $N$ is the set of all features in the causal graph.
\end{theorem}
\begin{proof}
By construction, our normalization step ensures:
\begin{align*}
    \sum_{i \in N} \phi_i^{\text{n}} &= \sum_{i \in N} \phi_i^{\text{c}} \times \frac{f(x) - \mathbb{E}[f(X)]}{\sum_{j=1}^n \phi_j^{\text{c}}} \\
    &= (f(x) - \mathbb{E}[f(X)]) \times \frac{\sum_{i \in N} \phi_i^{\text{c}}}{\sum_{j=1}^n \phi_j^{\text{c}}} \\
    &= f(x) - \mathbb{E}[f(X)].
\end{align*}
\end{proof}
\begin{theorem}[Missingness]
\label{thm:missingness}
Let $G = (V,E)$ be the causal graph and $i$ be a feature. If $i \notin
V$ or $i$ has no path to the target variable in $G$, then 
        $\phi_i^{\text{c}} = 0.$
\end{theorem}
\begin{proof}
This follows from two cases. Firstly, if $i \notin V$, then by
definition $W_i = 0$ and thus $\gamma_{i} = \frac{|W_i|}{\sum_{j \in
    N} |W_j|} = 0$. Secondly, if $i$ has no path to the target, then
$|W_i| = 0$ as there is no causal effect, again leading to $\gamma_{i}
= 0$. In both cases, $\phi_i^{\text{c}} = 0$. 
\end{proof}
\begin{theorem}[Consistency]
\label{thm:consistency}
Let $f_x(z') = f(h_x(z'))$, where $h_x$ maps simplified inputs to
original inputs, and let $z' \setminus i$ denote setting $z'_i =
0$. For any two models $f$ and $f'$, if: 
    \begin{equation}
        f'_x(z') - f'_x(z' \setminus i) \geq f_x(z') - f_x(z' \setminus i)
    \end{equation}
holds for all $z' \in \{0,1\}^M$, then:
    \begin{equation}
        \phi_i^\text{c}(f', x) \geq \phi_i^\text{c}(f, x),
    \end{equation}
where $\phi_i^\text{c}(\cdot, x)$ denotes the Causal SHAP value for
feature $i$ at point $x$.
\end{theorem}
\begin{proof}
  From Eqn~\ref{eqn:phi_i_c}, we have
\begin{equation*}
    \phi_i^{\text{c}}(f, x) = \sum_{S \subseteq N \setminus \{i\}} w(S,i)[f_x(S \cup \{i\}) - f_x(S)],
\end{equation*}
where $w(S,i) = \frac{|S|!(M-|S|-1)!}{M!} \gamma_{i}$ and $\gamma_{i}
= \frac{|W_i|}{\sum_{j \in N} |W_j|}$ is determined by the total
causal effect $W_i$
on the target.
Let $\Delta f_x(S) = f_x(S \cup \{i\}) - f_x(S)$ denote the marginal contribution of feature $i$ given subset $S$. Since $f$ and $f'$ share the same causal structure, $\gamma_{i}$ remains constant. Then:
\begin{align*}
    & \phi_i^{\text{c}}(f', x) - \phi_i^{\text{c}}(f, x) \\
    & = \sum_{S \subseteq N \setminus \{i\}} w(S,i)[\Delta f'_x(S) - \Delta f_x(S)] \\
    & = \sum_{S \subseteq N \setminus \{i\}} w(S,i)[\Delta f'_x - \Delta f_x] \\
    & \geq 0.
\end{align*}
The final inequality follows from: $w(S,i) \geq 0$ for all $S$ since
both the Shapley kernel and $\gamma_{i}$ are non-negative, $\Delta
f'_x \geq \Delta f_x$ by the hypothesis of the theorem, the causal
structure (and thus $\gamma_{i,s}$) remains unchanged between $f$ and
$f'$. Therefore, $\phi_i^{\text{c}}(f', x) \geq \phi_i^{\text{c}}(f,
x)$. 
\end{proof}
\noindent
Note that although Theorems~\ref{thm:missingness} and \ref{thm:consistency} are shown with respect to $\phi_i^c$, it is easy to see that they hold for $\phi_i^n$ as well. 

\section{Algorithm and Experiments}
We provide a detailed algorithm for computing Causal SHAP as shown in
Algorithm~\ref{alg:one}.
\begin{algorithm}
\caption{Causal SHAP Value Computation. \label{alg:one}}
\begin{algorithmic}[1]
\REQUIRE 
\STATE $f: \mathcal{X} \rightarrow \mathbb{R}$ : trained model
\STATE $x \in \mathcal{X}$ : instance to explain
\STATE $G = (V,E)$ : causal graph from PC algorithm  
\STATE $\{W_i\}_{i=1}^n$ : total causal effects from IDA algorithm
\STATE $T$ : number of Monte Carlo samples
\ENSURE Normalized causal Shapley values $\{\phi_i^{n}\}_{i=1}^n$

\STATE Compute causal weight factors: $\gamma_i \leftarrow \frac{|W_i|}{\sum_{j=1}^n |W_j|}$ for $i \in \{1,...,n\}$
\STATE Initialize: $\phi_i^c \leftarrow 0$ for $i \in \{1,...,n\}$
\STATE $\mathbb{E}[f(X)] \leftarrow \frac{1}{m}\sum_{j=1}^m f(x_j)$ where $\{x_j\}_{j=1}^m$ is training data

\FOR{$t = 1$ to $T$}
    \STATE $S \leftarrow$ uniform random subset of $\{1,...,n\}$
    \FOR{$i \in \{1,...,n\} \setminus S$}
        \STATE $v_c(S)$ via Monte Carlo sampling on $G$
        \STATE $v_c(S \cup \{i\})$ via Monte Carlo sampling on $G$
        \STATE $w \leftarrow \frac{|S|!(n-|S|-1)!}{n!} \times \gamma_i$
        \STATE $\phi_i^c \leftarrow \phi_i^c + w(v_c(S \cup \{i\}) - v_c(S))$
    \ENDFOR
\ENDFOR

\STATE Normalize: $\phi_i^n \leftarrow \phi_i^c \times \frac{f(x) - \mathbb{E}[f(X)]}{\sum_{j=1}^n \phi_j^c}$ for $i \in \{1,...,n\}$
\RETURN $\{\phi_i^n\}_{i=1}^n$
\end{algorithmic}
\end{algorithm}
As computing Shapley value is expensive, sampling methods are
widely used.
In our work, Monte Carlo sampling is used to approximate $v_{c}$ as follows:
\begin{equation}
    v_{c}(S) \approx \frac{1}{M}\sum_{m=1}^M f(x_S, \tilde{x}_{\bar{S}}^{(m)}),
    \label{eqn:MC}
\end{equation}
where $M$ is the number of Monte Carlo samples,
$\tilde{x}_{\bar{S}}^{(m)}$ is the $m$-th sample of non-intervened
features. Samples are generated through the following process. For
each feature $i \in \bar{S}$, in topological order: 
\begin{itemize}
    \item If feature has no parents: sample from its marginal distribution
    \item Otherwise: sample from $\mathcal{N}(\mu_i, \sigma_i)$ where: $\mu_i$ is the predicted mean from linear regression conditioned on parent values and $\sigma_i$ is estimated from regression residuals
\end{itemize}
In addition to using Monte Carlo, we use multiprocessing pool to
parallelize the Causal SHAP computation on AMD EPYC 7713 in our
experiments.


We evaluated our approach with four datasets (see
Table~\ref{table:dataset}). With the two synthetic datasets (one with
only direct causal effects and another one with both direct and indirect
causal effects), we assess how well each
method captures true causal relationships, using the known causal
graph as a reference (as the datasets are synthetic with known graph
ground truth). Secondly, SHAP is also applied to a reduced
feature set (a subset of the features that have a direct causal path
to the target variable, without being influenced by other features)
and served as a ground truth to evaluate how well each method captures
the true causal effect. For the two real-world datasets, we use the
{\em insertion score} \cite{petsiuk2018rise}, which is a standard
metric used for feature attribution evaluation in e.g.
\cite{covert2021explaining} and \cite{wang2020score} to evaluate the
performance.

For benchmarking, we compare our Causal SHAP method with 5 other
SHAP-derived methods: Independent SHAP\cite{lundberg2017unified},
Kernel SHAP\cite{lundberg2017unified}, On Manifold SHAP
\cite{frye2020shapley}, Shapley Flow\cite{wang2021shapley} and
Asymmetric Shapley Values (ASV) \cite{frye2020asymmetric}. Note that
other feature-attribution methods such as
LIME \cite{ribeiro2016should} are omitted due to our focus on
SHAP-derived methods.

\begin{table*}[ht]
\centering
\caption{Summary of datasets.\label{table:dataset}}
\begin{tabular}{lcccccc}
\toprule
\textbf{Dataset} & \textbf{\# Features} & \textbf{Train Size} & \textbf{Test Size} & \textbf{Problem} & \textbf{Data Type} \\
\midrule
Lung Cancer Risk & 4 & 800   & 200  & Regression & Synthetic \\
Cardiovascular Risk & 5 & 800 & 200 & Regression & Synthetic \\
Irritable bowel syndrome (IBS)   & 31  & 294 & 74  & Classification & Real-world \\
Colorectal Cancer  & 21 & 18321 & 4580 & Classification &  Real-world \\
\bottomrule
\end{tabular}
\end{table*}

\begin{table*}[h]
\centering
\caption{Comparison of SHAP Methods on Synthetic Dataset 1 \&
2. \textbf{Bold} values indicate the best performance
and \underline{underline} values indicate second best.}
\begin{scriptsize}
\begin{tabular}{|l|c|c|c|c|c|c|c|}
\hline
\textbf{Feature (Lung Cancer)} & \textbf{Ground Truth} & \textbf{Causal SHAP} & \textbf{Independent SHAP} & \textbf{Kernel SHAP} & \textbf{On Manifold} & \textbf{Shapley Flow} & \textbf{ASV} \\ \hline
Smoking  & 5.2171  & 5.3462     & 3.9467    & 3.2405  & 2.0104              & 2.1600             & 5.9366 \\ \hline
Stress & 0.2507    & 0.2600    & 0.1031  & 0.3620  & 1.8759              & 0.2600             & 1.9125 \\ \hline
Drink Coffee     & --   & 0.0000  & 1.8163    & 1.9514    & 1.9798         & 0.0000    & -0.0805 \\ \hline
\textbf{RMSE} & -- & \textbf{0.0167} & \underline{1.6357} & 3.9193   & 12.9241 & 8.8037   & 3.2792\\ \hline

\midrule 

\textbf{Feature (Cardiovascular Risk)} & \textbf{Ground Truth} & \textbf{Causal SHAP} & \textbf{Independent SHAP} & \textbf{Kernel SHAP} & \textbf{On Manifold} & \textbf{Shapley Flow} & \textbf{ASV} \\ \hline
Diet Score  & 0.5526 & 0.2397 & -0.5545 & -0.1910 & 1.3073 & 2.4000 & 0.0288 \\ \hline
Sleep Duration  & 3.6362 & 2.2416 & 0.1123 & 0.3247 & 1.1862 & 3.9300 & 0.0401 \\ \hline
BMI        & -- & 3.6056 & 4.1359 & 4.2996 & 1.3811 & 3.0900 & 4.9054 \\ \hline
Mental Health  & --  & 0.0000 & 1.3065 & 1.7032 & 1.3838 & 0.0000 & 0.2841 \\ \hline
\textbf{RMSE} & --   & \textbf{2.0422} & 13.6421 & 11.5176 & 6.5710 & \underline{3.4993} &  13.2048 \\ \hline
\end{tabular}
\end{scriptsize}
\label{tab:synthetic_shap_comparison}
\end{table*}

\subsubsection{Synthetic Datasets}
\paragraph{Lung Cancer}
Features in this dataset are \emph{smoking}, \emph{stress},
and \emph{drink\_coffee}, with \emph{lung\_cancer\_risk} as the target
variable. The \emph{stress} and \emph{smoking} are both generated using a
normal distribution $\mathcal{N}(5, 2)$. The relationships between the
variables are inspired by cancer
research \cite{zhou2021causal,zhang2020psychological,richards2015caffeine} and are defined as: 
\begin{align}
\label{eqn:lung_1}
    \text{lung\_cancer\_risk} &= 2 \cdot \text{smoking} +
    1.2 \cdot \text{stress} + \epsilon_{risk}, \\
\label{eqn:lung_2}    
    \text{drink\_coffee} &= 2 \cdot \text{smoking} + \text{stress} +  \epsilon_{coffee},
\end{align}
where $\epsilon_{risk} \sim \mathcal{N}(0, 3)$ and
$\epsilon_{coffee} \sim \mathcal{N}(0, 1)$ represent independent noise
terms sampled from normal distributions. The parameters in
Equations~\ref{eqn:lung_1} and \ref{eqn:lung_2} are chosen to generate
synthetic data that exhibits clear causal relationships, allowing us
to validate our causal discovery method by comparing the PC
algorithm's output against a known ground truth causal structure. The
causal graph successfully recovered by the PC algorithm from
this data set is shown in Figure~\ref{fig:graph_recovered}.

\begin{figure}[H]
    \centering
    \includegraphics[width=1\linewidth]{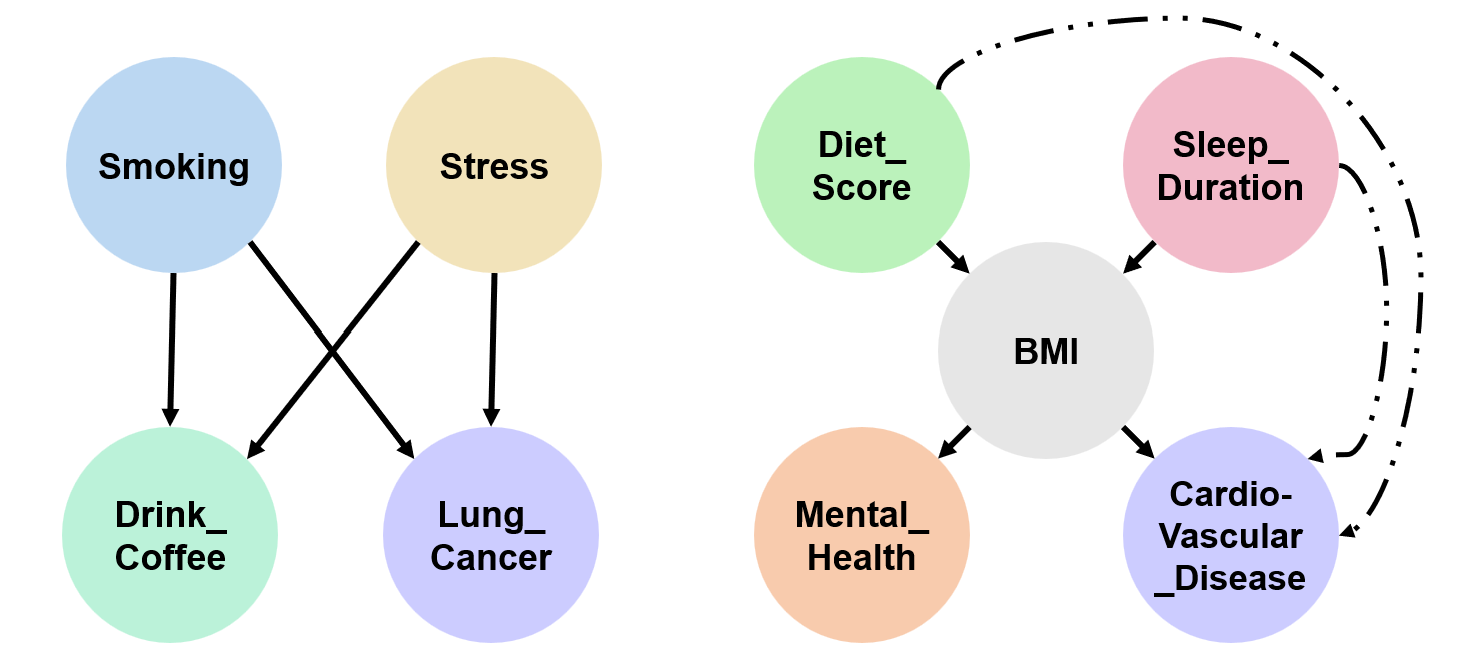}
    \caption{Diagram shows the causal graph of Lung Cancer Data set (Left), and the causal graph of Cardiovascular disease (Right). Solid arrow indicates causal edge with direct effect generated from PC algorithm, dashed lines represent the indirect effect. }
    \label{fig:graph_recovered}
\end{figure}

In the graph, \emph{smoking} and \emph{stress} act as confounders, influencing and creating a non-causal correlation between \emph{lung\_cancer\_risk} and \emph{drink\_coffee}. With this
design, we expect a causally-aware attribution method to assign
significant SHAP values to \emph{smoking} and \emph{stress}, as they
are direct causes, while assigning near-zero SHAP value
to \emph{drink\_coffee}, despite its correlation
with \emph{lung\_cancer\_risk}. This allows us to differentiate
between methods that only capture statistical associations versus
those that successfully incorporate causal relationships. 

In Figure~\ref{fig:graph_recovered}, we observe that Causal SHAP
attributes a zero SHAP score to \emph{drink\_coffee} in accordance with the causal graph. While Kernel SHAP misattributes significant scores
to \emph{drink\_coffee}, and fails to distinguish between causality and
correlation.

\begin{figure}[H]
    \centering
    \includegraphics[width=1\linewidth]{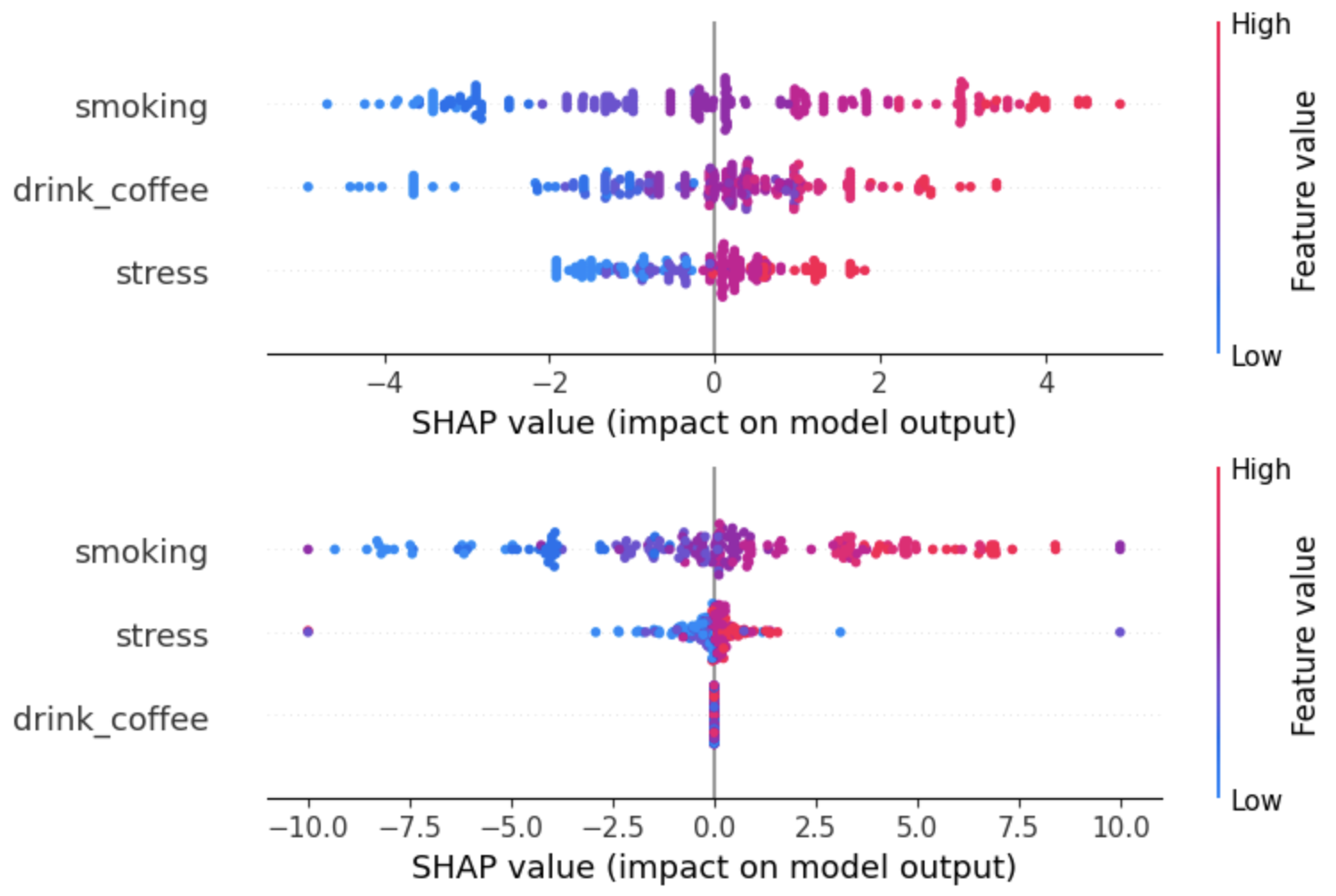}
    \caption{Kernel SHAP (top) versus Causal SHAP (bottom) for features in Lung Cancer Synthetic Dataset, \emph{drink\_coffee} in Causal SHAP is correctly attributed zero score as it is not affecting \emph{Lung\_Cancer\_Risk}}
    \label{fig:enter-label}
\end{figure}
Of all the methods listed in Table~\ref{tab:synthetic_shap_comparison}, Causal SHAP achieved the lowest Root Mean Squared Error (RMSE) against ground truth, best capturing the true causal effects. Independent SHAP, Kernel SHAP, and On Manifold SHAP failed to capture causality. Shapley Flow and ASV successfully give a near-zero score to $drink\_coffee$, but unfortunately have a larger discrepancy in score on other variables.

\paragraph{Cardiovascular Diseases}
This dataset has \emph{diet\_score}, \emph{sleep\_duration},
and \emph{family\_history} as input. The first is
randomly generated from a uniform distribution $\mathcal{U}(1,10)$, while the latter two come from $\mathcal{N}(8, 4)$ and
$\mathcal{N}(4, 2)$ distributions.
\begin{align}
\label{eqn:cardio_1}
    \text{bmi} &= 0.4 \cdot \text{diet\_score} +
    0.5 \cdot \text{sleep\_duration} + \epsilon_{bmi}, \\
\label{eqn:cardio_2}    
    \text{mental\_health} &= 1.5 \cdot \text{bmi}
    + \epsilon_{health}, \\
\label{eqn:cardio_3}    
    \text{cv\_risk} &= 1.5 \cdot \text{bmi} + \epsilon_{risk}, 
\end{align}
\begin{figure} [H]
    \centering
    \includegraphics[width=1\linewidth]{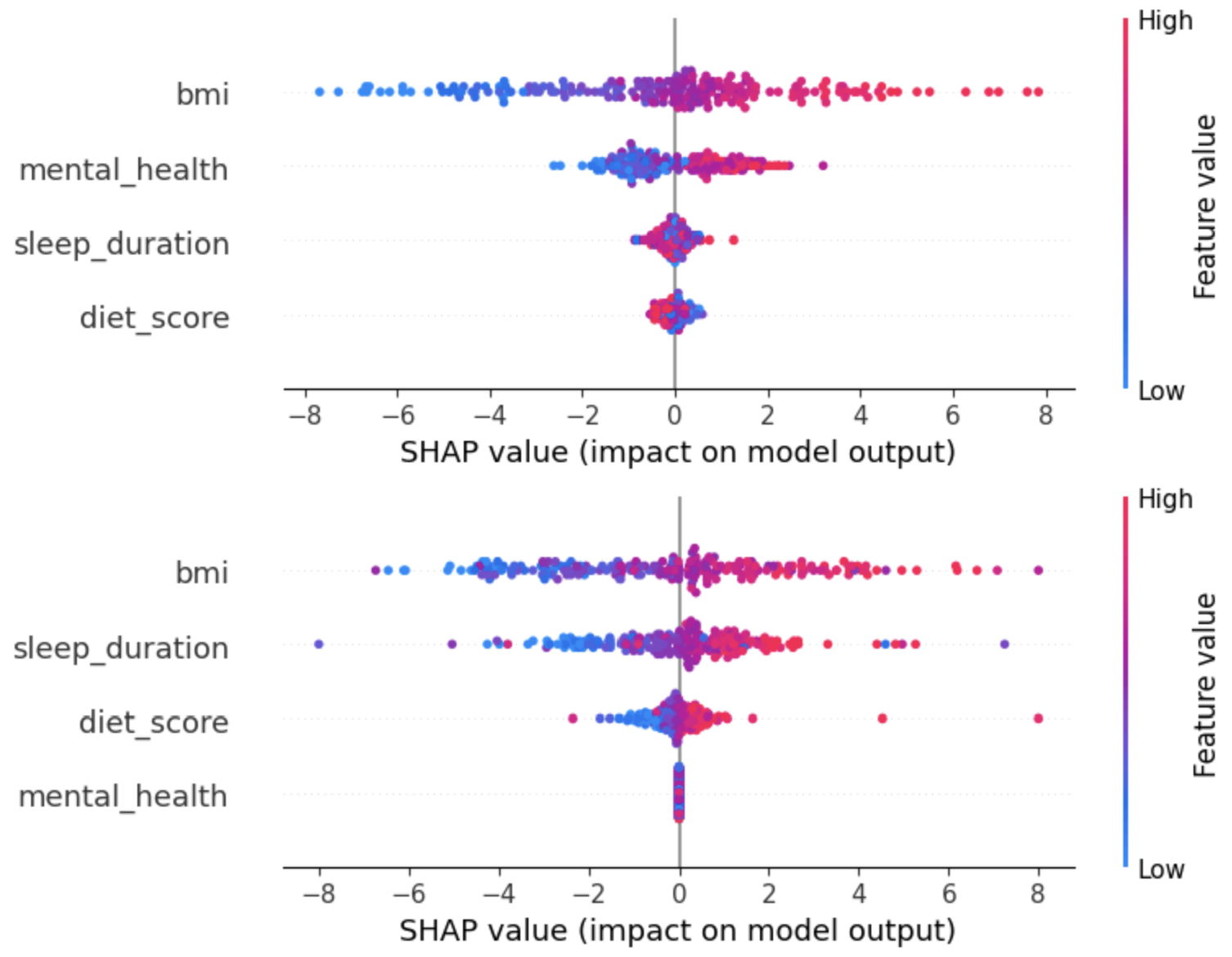}
    \caption{Kernel SHAP (top) versus Causal SHAP (bottom) for features in Cardiovascular Disease Synthetic Dataset, \emph{mental\_health} is correctly attributed zero score while \emph{diet\_score} and \emph{sleep\_duration} are assigned higher scores due to their indirect effect to \emph{CV\_risk}.}
    \label{fig:cardio_result}
\end{figure}
\noindent
where $\epsilon_{bmi} \sim \mathcal{N}(0,1)$,
    $\epsilon_{health} \sim \mathcal{N}(0,1)$ and
    $\epsilon_{risk} \sim \mathcal{N}(2,3)$ are independent noise 
    added to the features. The parameters in
    Eqns~\ref{eqn:cardio_1}-\ref{eqn:cardio_3} are inspired
    by \cite{wolongevicz2010diet}\cite{theorell2016sleep}\cite{luppino2010overweight}\cite{ortega2016obesity}
    and chosen to generate synthetic data with clear causal
    relationships. Figure~\ref{fig:graph_recovered} shows the causal graph recovered by the PC algorithm. 

Similar to the previous dataset, we have \emph{bmi} that acts as our
confounder. We expect a causally-aware attribution method to assign
significant SHAP values to \emph{bmi} as it is the direct cause, at
the same time assign a lower but non-zero SHAP values
to \emph{diet\_score} and \emph{sleep\_duration} as they are the
indirect causes. 

In Figure~\ref{fig:cardio_result}, Causal SHAP successfully attributes
zero score to \emph{mental\_health} and also provides a higher
importance to \emph{sleep\_duration} and \emph{diet\_score} because
these are the two factors that causally affect \emph{bmi}. While
Kernel SHAP failed to distinguish correlation and causation, it did
not manage to capture the causal links of \emph{sleep\_duration}
nor \emph{diet\_score} thus assigned near zero scores to them.  

Referring to Table 2, Causal SHAP achieved the lowest RMSE compared to Ground Truth, best capturing the true causal effects of \emph{diet\_score} and \emph{sleep\_duration}. Note that \emph{mental\_health} is omitted as it lacks a path to \emph{CV\_risk}, and \emph{bmi} is excluded because its causal relationship with two other variables, which may make it difficult
to interpret the Ground Truth values for BMI in terms of the true
causal effects. 

\subsubsection{Real-World Datasets}
We evaluated our method on two biomedical datasets: IBS \cite{jacobs2023multi} and Colorectal cancer datasets \cite{zhang2024apoe}, selected for their high complexity of causal links between metabolites, lifestyle factors, physiological measurements, and genetic information. 
We compared Causal SHAP against five baseline methods using the insertion test \cite{petsiuk2018rise}, which sequentially adds features from most to least important based on each method's attributions, measuring AUROC, Cross Entropy and Brier scores at each step. High AUROC, lower Cross Entropy and Brier scores indicate better performance. All experiments used the same Random Forest model as the black box.

As shown in Table~\ref{tab:main-results}, Causal SHAP achieved: 
\begin{itemize}
    \item On IBS dataset: Best AUROC (0.8594) and second-best Cross Entropy (0.4645) and Brier scores (0.1464)
    \item On Colorectal Cancer dataset: Best performance across all metrics (AUROC: 0.6271, Cross Entropy: 0.6735, Brier: 0.2397)
\end{itemize}

While Shapley Flow and ASV incorporate causal information via PC and IDA algorithms, Causal SHAP's superior performance suggests it leverages causal relationships more effectively. On Manifold performed strongly with IBS data, possibly because its dependency-agnostic approach works better when causal relationships are uncertain.

\begin{table}[!htbp]  
\footnotesize
\setlength{\tabcolsep}{2pt}
\begin{threeparttable}
\begin{minipage}{\textwidth}
\centering
\caption{Main Results: Comparison of Global Feature Attribution Methods}
\label{tab:main-results}
\begin{tabular}{@{}lccc@{}}
\toprule
& \multicolumn{3}{c}{Insertion Score} \\
\cmidrule(lr){2-4} 
& AUROC & Cross Entropy & Brier  \\
\midrule
IBS & & &\\
- Independent SHAP & 0.8527 ± 0.0326 & 0.4779 ± 0.0336 & 0.1517 ± 0.0136  \\
- Kernel SHAP & 0.8548 ± 0.0284 & 0.4674 ± 0.0360 & 0.1474 ± 0.0148  \\
- On Manifold & \underline{0.8589 ± 0.0244} & \textbf{0.4644 ± 0.0311} & \textbf{0.1461 ± 0.0132} \\
- Shapley Flow & 0.8274 ± 0.0359 & 0.5233 ± 0.0434 & 0.1706 ± 0.0178  \\
- ASV & 0.8225 ± 0.0356 & 0.5327 ± 0.0414 & 0.1756 ± 0.0174  \\
- Causal SHAP & \textbf{0.8594 ± 0.0227} & \underline{0.4645 ± 0.0314} & \underline{0.1464 ± 0.0134} \\
Colorectal Cancer & & \\
- Independent SHAP & 0.5886 ± 0.0337 & 0.6893 ± 0.0175 & 0.2475 ± 0.0080  \\
- Kernel SHAP & \underline{0.6263 ± 0.0149} & 0.6766 ± 0.0090 & 0.2410 ± 0.0040  \\
- On Manifold & 0.6243 ± 0.0132 & 0.6781 ± 0.0081 & 0.2416 ± 0.0035  \\
- Shapley Flow & 0.5868 ± 0.0262 & 0.6864 ± 0.0148 & 0.2463 ± 0.0069  \\
- ASV & 0.6177 ± 0.0170 & \underline{0.6750 ± 0.0087} & \underline{0.2406 ± 0.0039}  \\
- Causal SHAP & \textbf{0.6271 ± 0.0190} & \textbf{0.6735 ± 0.0096} & \textbf{0.2397 ± 0.0044} \\   
\midrule
\bottomrule
\end{tabular}
\small
\begin{tablenotes}
\item[1] All results are averaged over 5 runs with different random seeds. $\mu$ represents the mean and $\sigma$ represents the standard deviation.
\item[2] Bold values indicate the best performance and underline values indicate second best for each dataset and metric.
\end{tablenotes}
\end{minipage}
\end{threeparttable}
\end{table}

Computationally, we studied the impact of the hyperparameter $M$ in Equation~\ref{eqn:MC}. We find that AUROC is not sensitive to changes in $M$ once it exceeds 64. (e.g., IBS dataset: M=32 yields AUROC=0.5678; M=64 yields AUROC=0.5680). We also analyzed the computational cost of different components in our pipeline. To compute feature attributions for all IBS instances, the Causal SHAP Value Computation, PC, and IDA take 366.13, 0.26, and 1.56 seconds, respectively.
\section{Conclusion and Future Work}
In this paper, we introduced Causal SHAP, a comprehensive feature
attribution explainable AI method that considers causality between
input features. Causal SHAP leverages the PC algorithm for causal path
discovery and the IDA algorithm to estimate causal strength. Through
experiments on multiple synthetic and real-world datasets, we
demonstrated several key advantages of our method:
\begin{enumerate}
  \item
It effectively respects causal relationships, even in the presence of
highly correlated features.
\item
It achieves superior insertion scores on real-world datasets, even
when the true causal graph is unknown.
\item
  It preserves key theoretical properties of SHAP.
\end{enumerate}

We chose the PC algorithm due to its widespread use, under the
assumption that no hidden variables influence the
dataset. In future work, we aim to (1) 
address cases with hidden variables using the FCI algorithm
\cite{richardson2013discovery} for causal graph estimation; (2)
Explore Greedy Equivalence Search (GES)
\cite{meek1997graphical} when the PC algorithm becomes computationally
expensive as number of features scale; (3) Enhance our
framework to handle causal structure uncertainty in by aggregating
multiple possible causal graphs. (4) Develop more efficient
approximation algorithms for high-dimensional datasets with complex
dependency structures.

\section*{Acknowledgement}
This research is supported by the Ministry of Education, Singapore
(Grant IDs: RG22/23, RS15/23, LKCMedicine Start up Grant).

\bibliographystyle{IEEEtran}

\begin{thebibliography}{99}
\bibitem{vilone2021notions}Vilone, G. \& Longo, L. Notions of explainability and evaluation approaches for explainable artificial intelligence. {\em Information Fusion}. \textbf{76} pp. 89-106 (2021)
\bibitem{lundberg2017unified}Lundberg, S. A unified approach to interpreting model predictions. {\em ArXiv Preprint ArXiv:1705.07874}. (2017)
\bibitem{shapley1953value}Shapley, L. A value for n-person games. {\em Contribution To The Theory Of Games}. \textbf{2} (1953)
\bibitem{michalak2013efficient}Michalak, T., Aadithya, K., Szczepanski, P., Ravindran, B. \& Jennings, N. Efficient computation of the Shapley value for game-theoretic network centrality. {\em Journal Of Artificial Intelligence Research}. \textbf{46} pp. 607-650 (2013)
\bibitem{frye2020asymmetric}Frye, C., Rowat, C. \& Feige, I. Asymmetric shapley values: incorporating causal knowledge into model-agnostic explainability. {\em Advances In Neural Information Processing Systems}. \textbf{33} pp. 1229-1239 (2020)
\bibitem{spirtes2001causation}Spirtes, P., Glymour, C. \& Scheines, R. Causation, prediction, and search. (MIT press,2001)
\bibitem{maathuis2010predicting}Maathuis, M., Colombo, D., Kalisch, M. \& Bühlmann, P. Predicting causal effects in large-scale systems from observational data. {\em Nature Methods}. \textbf{7}, 247-248 (2010)
\bibitem{heskes2020causal}Heskes, T., Sijben, E., Bucur, I. \& Claassen, T. Causal shapley values: Exploiting causal knowledge to explain individual predictions of complex models. {\em Advances In Neural Information Processing Systems}. \textbf{33} pp. 4778-4789 (2020)
\bibitem{ribeiro2016should}Ribeiro, M., Singh, S. \& Guestrin, C. " Why should i trust you?" Explaining the predictions of any classifier. {\em Proceedings Of The 22nd ACM SIGKDD International Conference On Knowledge Discovery And Data Mining}. pp. 1135-1144 (2016)
\bibitem{wang2021shapley}Wang, J., Wiens, J. \& Lundberg, S. Shapley flow: A graph-based approach to interpreting model predictions. {\em International Conference On Artificial Intelligence And Statistics}. pp. 721-729 (2021)
\bibitem{janzing2020feature}Janzing, D., Minorics, L. \& Blöbaum, P. Feature relevance quantification in explainable AI: A causal problem. {\em International Conference On Artificial Intelligence And Statistics}. pp. 2907-2916 (2020)
\bibitem{aas2021explaining}Aas, K., Jullum, M. \& Løland, A. Explaining individual predictions when features are dependent: More accurate approximations to Shapley values. {\em Artificial Intelligence}. \textbf{298} pp. 103502 (2021)
\bibitem{lundberg2020local}Lundberg, S. \& Others From local explanations to global understanding with explainable AI for trees. {\em Nature Machine Intelligence}. \textbf{2}, 56-67 (2020)
\bibitem{amoukou2022accurate}Amoukou, S., Salaün, T. \& Brunel, N. Accurate shapley values for explaining tree-based models. {\em International Conference On Artificial Intelligence And Statistics}. pp. 2448-2465 (2022)
\bibitem{petsiuk2018rise}Petsiuk, V. Rise: Randomized Input Sampling for Explanation of black-box models. {\em ArXiv Preprint ArXiv:1806.07421}. (2018)
\bibitem{richardson2013discovery}Richardson, T. A discovery algorithm for directed cyclic graphs. {\em ArXiv Preprint ArXiv:1302.3599}. (2013)
\bibitem{meek1997graphical}Meek, C. Graphical Models: Selecting causal and statistical models. (Carnegie Mellon University,1997)
\bibitem{kalisch2012causal}Kalisch, M., Mächler, M., Colombo, D., Maathuis, M. \& Bühlmann, P. Causal inference using graphical models with the R package pcalg. {\em Journal Of Statistical Software}. \textbf{47} pp. 1-26 (2012)
\bibitem{frye2020shapley}Frye, C., Mijolla, D., Begley, T., Cowton, L., Stanley, M. \& Feige, I. Shapley explainability on the data manifold. {\em ArXiv Preprint ArXiv:2006.01272}. (2020)
\bibitem{jacobs2023multi}Jacobs, J. \& Others Multi-omics profiles of the intestinal microbiome in irritable bowel syndrome and its bowel habit subtypes. {\em Microbiome}. \textbf{11}, 5 (2023)
\bibitem{zhang2024apoe}Zhang, J. \& Others APOE genotype modifies the association between midlife adherence to the planetary healthy diet and cognitive function in later life among Chinese adults in Singapore. {\em The Journal Of Nutrition}. \textbf{154}, 252-260 (2024)
\bibitem{pearl2000models}Pearl, J. \& Others Models, reasoning and inference. {\em Cambridge, UK: CambridgeUniversityPress}. \textbf{19}, 3 (2000)
\bibitem{covert2021explaining}Covert, I., Lundberg, S. \& Lee, S. Explaining by removing: A unified framework for model explanation. {\em Journal Of Machine Learning Research}. \textbf{22}, 1-90 (2021)
\bibitem{wang2020score}Wang, H. \& Others Score-CAM: Score-weighted visual explanations for convolutional neural networks. {\em Proceedings Of The IEEE/CVF Conference On Computer Vision And Pattern Recognition Workshops}. pp. 24-25 (2020)
\bibitem{zhou2021causal}Zhou, W. \& Others Causal relationships between body mass index, smoking and lung cancer: univariable and multivariable Mendelian randomization. {\em International Journal Of Cancer}. \textbf{148}, 1077-1086 (2021)
\bibitem{hooker2019please}Hooker, G. \& Mentch, L. Please stop permuting features: An explanation and alternatives. {\em ArXiv Preprint ArXiv:1905.03151}. \textbf{2} pp. 1 (2019)
\bibitem{zhang2020psychological}Zhang, Y. \& Others Psychological stress enhances tumor growth and diminishes radiation response in preclinical model of lung cancer. {\em Radiotherapy And Oncology}. \textbf{146} pp. 126-135 (2020)
\bibitem{richards2015caffeine}Richards, G. \& Smith, A. Caffeine consumption and self-assessed stress, anxiety, and depression in secondary school children. {\em Journal Of Psychopharmacology}. \textbf{29}, 1236-1247 (2015)
\bibitem{wolongevicz2010diet}Wolongevicz, D. \& Others Diet quality and obesity in women: the Framingham Nutrition Studies. {\em British Journal Of Nutrition}. \textbf{103}, 1223-1229 (2010)
\bibitem{theorell2016sleep}Theorell-Haglöw, J. \& Lindberg, E. Sleep duration and obesity in adults: what are the connections?. {\em Current Obesity Reports}. \textbf{5} pp. 333-343 (2016)
\bibitem{luppino2010overweight}Luppino, F., Wit, L., Bouvy, P., Stijnen, T., Cuijpers, P., Penninx, B. \& Zitman, F. Overweight, obesity, and depression: a systematic review and meta-analysis of longitudinal studies. {\em Archives Of General Psychiatry}. \textbf{67}, 220-229 (2010)
\bibitem{ortega2016obesity}Ortega, F., Lavie, C. \& Blair, S. Obesity and cardiovascular disease. {\em Circulation Research}. \textbf{118}, 1752-1770 (2016)

\end{thebibliography}

\end{document}